%% file: fastep_arxiv.tex
\documentclass[11pt]{article}
\usepackage{natbib}
\usepackage{amsmath}
\usepackage{amsfonts}
\usepackage{amsopn}
\usepackage{ifthen}
\usepackage{algorithm}
\usepackage{algorithmic}
\usepackage{bbm}
\usepackage{bm}
\usepackage{graphicx}
\usepackage{url}
\usepackage{color}

\input{mymacros.tex} 
\input{mymacros2.tex}

\definecolor{gray}{rgb}{.5,.5,.5}
\newtheorem{theo}{Theorem}

\newcommand{\BlackBox}{\rule{1.5ex}{1.5ex}}  
\newenvironment{proof}{\par\noindent{\bf Proof\ }}{\hfill\BlackBox\\[2mm]}

\title{Fast Convergent Algorithms for Expectation Propagation Approximate Bayesian Inference}
\author{Matthias W. Seeger\thanks{Probabilistic Machine Learning Laboratory, Ecole Polytechnique F\'{e}d\'{e}rale de Lausanne, INJ 339, Station 14, CH-1015 Lausanne, ({\tt matthias.seeger@epfl.ch}).}
        \and Hannes Nickisch \thanks{Max Planck Institute for Biological Cybernetics, Spemannstra{\ss}e 38, 72076 T\"{u}bingen, ({\tt hn@tuebingen.mpg.de}).}}

\begin{document}

\maketitle

\begin{abstract}%
We propose a novel algorithm to solve the expectation propagation relaxation
of Bayesian inference for continuous-variable graphical models. In contrast
to most previous algorithms, our method is provably convergent. By marrying
convergent EP ideas from \cite{Opper:05} with covariance decoupling techniques
\cite{Wipf:08,Nickisch:09}, it runs at least an order of magnitude faster than
the most commonly used EP solver.
\end{abstract}

\section{Introduction}\label{sec:intro}

A growing number of challenging machine learning applications require
decision-making from incomplete data (e.g., stochastic optimization, active
sampling, robotics), which relies on quantitative representations of
uncertainty (e.g., Bayesian posterior, belief state) and is out of reach of
the commonly used paradigm of learning as point estimation on hand-selected
data. While Bayesian inference is harder than point estimation in general, it
can be relaxed to {\em variational} optimization problems which can be
computationally competitive, if only they are treated with the algorithmic
state-of-the-art established for the latter.

In this paper, we propose a novel algorithm for the expectation
propagation (EP; or adaptive TAP, or expectation consistent (EC)) relaxation
\cite{Opper:01,Minka:01a,Opper:05}, which is both much faster than the
commonly used sequential EP algorithm, and is provably convergent (the
sequential algorithm lacks such a guarantee). Our method builds on the
convergent double loop algorithm of \cite{Opper:05}, but runs orders of
magnitude faster. We gain a deeper understanding of EP (or EC) as optimization
problem, unifying it with covariance decoupling ideas
\cite{Wipf:08,Nickisch:09}, and allowing for ``point estimation'' algorithmic
progress to be brought to bear on this powerful approximate inference
formulation.

Suppose that observations $\vy{}\in\R^m$ are modelled as $\vy{} =
\mxx{}\vu{}+\veps{}$, where $\vu{}\in\R^n$ are latent variables of interest,
$\veps{}\sim N(\vzero,\sigma^2\Id)$ is Gaussian noise, and $\mxx{}\in
\R^{m\times n}$ is the design matrix. For example, $\vu{}$ can be an
image to be reconstructed from $\vy{}$ (e.g., Fourier coefficients in magnetic
resonance imaging \cite{Seeger:08a}), further examples are found in
\cite{Seeger:07d}. The prior distribution has the form
$P(\vu{})\propto \prod_{i=1}^q t_i(s_i)$ with {\em non-Gaussian} potentials
$t_i(\cdot)$, and $\vs{}:=\mxb{}\vu{}$ for a matrix $\mxb{}$. A well-known
example are {\em Laplace sparsity priors} defined by $t_i(s_i) =
e^{-\tau_i|s_i|}$ \cite{Seeger:07d}, where $\mxb{}$ collects simple filters
(e.g., derivatives, wavelet coefficients). This formal setup also encompasses
binary classification ($\vu{}$ classifier weights,
$\prod_{i=1}^q t_i(s_i)$ the classification likelihood \cite{Nickisch:09})
or spiking neuron models \cite{Gerwinn:08}. The posterior distribution is
\begin{equation}\label{eq:posterior}
  P(\vu{}|\vy{}) = Z^{-1} N(\vy{}|\mxx{}\vu{},\sigma^2\Id)
  \prod\nolimits_{i=1}^q t_i(s_i),
\end{equation}
$Z := \int N(\vy{}|\mxx{}\vu{},\sigma^2\Id)\prod_{i=1}^q t_i(s_i)\,
d\vu{}$ the {\em partition function} for $P(\vu{}|\vy{})$, and
$\vs{}=\mxb{}\vu{}$. {\em Bayesian inference} amounts to computing moments of
$P(\vu{}|\vy{})$ and/or $\log Z$. Hyperparameters $\vf{}$ can be learned by
maximizing $\log Z(\vf{})$ \cite{MacKay:03} (e.g., motion deblurring by blind
deconvolution \cite{Levin:09}). In Bayesian
experimental design (or active learning) \cite{Nickisch:09}, $\mxx{}$ is built
up sequentially by greedily maximizing expected information scores. These
applications require posterior covariance information beyond any single point
estimate.

The expectation propagation relaxation along with known algorithms is
described in \secref{expprop}, scalable inference techniques reviewed in
\secref{scalable}. We develop our novel algorithm in \secref{scal-ep},
provide a range of real-world experiments (image deblurring and reconstruction)
in \secref{exper}, and close with a discussion (\secref{discuss}). Upon
publication, code for our algorithm will be released into the public domain.

\section{Expectation Propagation}
\label{sec:expprop}

Expectation propagation (EP) \cite{Minka:01a,Opper:05} stands out among
variational inference approximations. First, it is more generally applicable
than most others (see end of \secref{scalable}). Second, a range of empirical
studies indicate that EP can be a far more accurate approximation to Bayesian
inference than today's competitors of comparable running time
\cite{Kuss:05,Nickisch:08}. Consequently, EP has been applied to a diverse
range of models.\footnote{
  A comprehensive bibliography can be found at
  \url{research.microsoft.com/en-us/um/people/minka/papers/ep/roadmap.html}.}
On the other hand, EP is more difficult to handle than most other methods, for
a number of reasons. It is not an optimization problem based on a bound on
$\log Z$ \eqp{posterior}, but constitutes a search for a saddle point
\cite{Opper:05}. Moreover, its stationary equations
are more complicated in structure than commonly used bounds.
Finally, running EP can be numerically challenging \cite{Seeger:07d,Barber:06}.

In the sequel, we describe the variational optimization problem behind
(fractional) EP, details can be found in
\cite{Minka:01a,Opper:05,Seeger:07d}. The goal is to fit the posterior
distribution $P(\vu{}|\vy{})$ from \eqp{posterior} by a
{\em Gaussian}
of the form
\begin{eqnarray}
  Q(\vu{}|\vy{}) := Z_Q^{-1} N(\vy{}|\mxx{}\vu{},\sigma^2\Id) e^{\vb{}^T\vs{}
  -\frac{1}2\vs{}^T(\diag\vpi{})\vs{}}, \nonumber \\
  \Cov_Q[\vu{}|\vy{}]^{-1} = \mxa{} := \sigma^{-2}\mxx{}^T\mxx{} +
  \mxb{}^T(\diag\vpi{})\mxb{}, \label{eq:post-approx}
\end{eqnarray}
where $Z_Q := \int N(\vy{}|\mxx{}\vu{},\sigma^2\Id) e^{\vb{}^T\vs{}
-\frac{1}2\vs{}^T(\diag\vpi{})\vs{}}\, d\vu{}$, $\vs{}=\mxb{}\vu{}$.
$Q(\vu{}|\vy{})$ depends on the variational parameters
$\vb{}$ and $\vpi{}\succeq\vzero$, collected as $\vth{} = (\vpi{},\vb{})$
below. Let marginal distributions $N(\mu_i,\rho_i)$ be indexed by moment
parameters $\vmu{}$, $\vrho{}$, $\eta\in(0,1]$ a fractional parameter
(while standard EP uses $\eta=1$, $\eta<1$ can
strongly improve numerical stability \cite{Seeger:07d}). For $i\in\srng{q}$,
denote $\kappa_i = \kappa_i(s_i) := b_i s_i - \frac{1}2 \pi_i s_i^2$. The
{\em cavity marginal} is $Q_{-i}(s_i) \propto N(s_i|\mu_i,\rho_i)
e^{-\eta\kappa_i}$, the {\em tilted marginal} $\hat{P}_i(s_i) \propto
Q_{-i}(s_i) t_i(s_i)^{\eta}$. While $\hat{P}_i(s_i)$ is not a Gaussian,
its moments (mean and variance) can be computed tractably. An EP fixed point
$(\vpi{},\vb{})$ satisfies {\em expectation consistency} \cite{Opper:05}: if
$N(\mu_i,\rho_i) = Q(s_i|\vy{})$, then $\hat{P}_i(s_i)$ and $Q(s_i|\vy{})$ have
the same mean and variance for all $i=\rng{q}$. The corresponding (negative
free) energy function is
\[
\begin{split}
  & \phi(\vpi{},\vb{},\vmu{},\vrho{}) := -2\log Z_Q \\
  & - {\textstyle\frac{2}{\eta}}\sum\nolimits_{i=1}^q
  \left( \log\Ex_{Q_{-i}}[t_i(s_i)^{\eta}] - \log\Ex_{Q_{-i}}[e^{\eta\kappa_i}]
  \right),
\end{split}
\]
where $Z_Q$ is the partition function of $Q(\vu{}|\vy{})$ (see
Eq.~\ref{eq:post-approx}). If
we define $\vmu{}$, $\vrho{}$ in terms of $\vpi{}$, $\vb{}$ (by requiring
that $N(\mu_i,\rho_i) = Q(s_i|\vy{})$), it is easy to see that
$\nabla_{\vpi{}}\phi = \nabla_{\vb{}}\phi = \vzero$ implies expectation
consistency. However, this dependency tends to be broken intermediately in
most EP algorithms. A schematic overview of the expectation consistency
conditions is as follows (notations $\tvth{}, \vth{-}, \vs{*}, \vz{}$ are
introduced in subsequent sections; $\stackrel{\text{MM}}{\longleftrightarrow}$
denotes Gaussian moment matching):
\begin{equation}\label{eq:ep-schematic}
  \begin{array}{ccc}
  \overbrace{N(\mu_i,\rho_i)}^{\tvth{}\leftrightarrow(\vmu{},\vrho{})} & \to &
  \overbrace{Q_{-i}(s_i)\propto N(s_i|\mu_i,\rho_i)
  e^{-\eta\kappa_i}}^{\vth{-} (=\tvth{}-\eta\vth{})} \\
  & & \downarrow \\
  \underbrace{Q(s_i|\vy{})}_{=N(s_{* i},z_i)} &
  \stackrel{\text{MM}}{\longleftrightarrow} & \hat{P}_i(s_i) \propto
  Q_{-i}(s_i) t_i(s_i)^{\eta}
  \end{array}
\end{equation}
The total criterion $\phi(\vpi{},\vb{},\vmu{}(\vpi{},\vb{}),
\vrho{}(\vpi{},\vb{}))$ is neither convex nor concave \cite{Opper:05}.

The most commonly used {\em sequential} EP algorithm visits each potential
$i\in\srng{q}$ in turn, first updating $\mu_i$, $\rho_i$, then $\pi_i$, $b_i$
based on one iteration\footnote{
  ``One iteration'' means solving for $\pi_i$, $b_i$, assuming that the cavity
  distribution $Q_{-i}(s_i)$ is fixed (ignoring its dependence on
  $\pi_i$, $b_i$).}
of $\partial_{\pi_i}\phi = \partial_{b_i}\phi = 0$ \cite{Minka:01a,Opper:05}.
For models of moderate size $n$, a numerically robust implementation maintains
the inverse covariance matrix $\mxa{}$ \eqp{post-approx} as
representation of $Q(\vu{}|\vy{})$. A sweep over all potentials costs
$O(q\, n^2)$. If memory costs of $O(n^2)$ are prohibitive, we can determine
$\mu_i$, $\rho_i$ on demand by solving a linear system with $\mxa{}$, in
which case a sweep requires $q$ such systems. The sequential EP algorithm is
too slow to be useful for many applications. Notably, all publications for EP
we are aware of (with the exception of two references discussed in the
sequel) employ this method, generally known as ``the EP algorithm''.

In \cite{Gerven:10}, a {\em parallel} variant of EP is applied to rather large
models of a particular structure. They alternate between updates of all
$\vmu{}$, $\vrho{}$ and all $\vpi{}$, $\vb{}$, the latter by one iteration of
$\partial_{\vpi{}}\phi = \partial_{\vb{}}\phi = \vzero$ (these equations
decouple w.r.t.\ $i=\rng{q}$). The most expensive step per iteration by far is
the computation of marginal variances $\vrho{}$, which is feasible only for the
very sparse
matrices $\mxa{}$ specific to their application. Neither sequential nor
parallel algorithm come with a convergence proof.

A provably convergent double loop algorithm for EP is given by Opper\&Winther
in \cite{Opper:05}. For its derivation, we need to consider a {\em natural
parameterization} of the problem. The underlying reason for this is that
log partition functions like $\log Z_Q$ \eqp{post-approx} are simple convex
functions in natural parameters, and derivatives w.r.t.\ the latter result
in posterior expectations. Collect $\vth{} = (\vpi{},\vb{})$ and recall that
$\kappa_i = b_i s_i - \frac{1}2 \pi_i
s_i^2$. Let $\tvth{}=(\tvpi{},\tvb{})$ be natural
parameters corresponding to $\vmu{}, \vrho{}$ ($\tilde{\pi}_i=1/\rho_i$,
$\tilde{b}_i=\mu_i/\rho_i$), and $\tilde{\kappa}_i = \tilde{b}_{i} s_i -
\frac{1}2\tilde{\pi}_{i} s_i^2$, so that $N(s_i|\mu_i,\rho_i) = Z_i^{-1}
e^{\tilde{\kappa}_i}$, where $Z_i = \int e^{\tilde{\kappa}_i}\, d s_i$ is the
normalization constant. With
$\vth{-} = (\vpi{-},\vb{-}) = \tvth{}-\eta\vth{}$ and $\kappa_{-i} = b_{-i} s_i
- \frac{1}2\pi_{-i} s_i^2 = \tilde{\kappa}_i - \eta\kappa_i$, we have that
$Q_{-i}(s_i)\propto e^{\kappa_{-i}}$
and $\hat{P}_i(s_i) = \hat{Z}_i^{-1} e^{\kappa_{-i}} t_i(s_i)^{\eta}$ with
$\hat{Z}_i = \int e^{\kappa_{-i}} t_i(s_i)^{\eta}\, d s_i$. If
$\phi_{\cap}(\vth{-},\tvth{}) := -\frac{2}{\eta}\sum_i \log\hat{Z}_i -2\log Z_Q$
and $\phi_{\cup}(\tvth{}) := \frac{2}{\eta}\sum_i \log Z_i$, we have that
$\phi(\vth{-},\tvth{}) = \phi_{\cap}(\vth{-},\tvth{}) + \phi_{\cup}(\tvth{})$,
where $\phi_{\cap}(\vth{-},\tvth{})$ is jointly concave\footnote{
  Log partition functions ($\log\hat{Z}_i$, $\log Z_Q$) are convex in their
  natural parameters, and $\vth{} = \eta^{-1}(\tvth{}-\vth{-})$ is linear.},
while $\phi_{\cup}(\tvth{})$ is convex. Define $\phi(\tvth{}) := \max_{\vth{-}}
\phi(\vth{-},\tvth{})$. The Opper\&Winther algorithm (locally) minimizes $\phi(\tvth{})$
via two nested loops. The inner loop (IL) is the concave maximization
$\vth{-}\leftarrow\argmax_{\vth{-}}\phi_{\cap}(\vth{-},\tvth{})$ for fixed
$\tvth{}$. An outer loop (OL) iteration consists of an IL followed by an update
of $\tvth{}$: $\vmu{}\leftarrow\Ex_Q[\vs{}|\vy{}]$, $\vrho{}\leftarrow
\Var_Q[\vs{}|\vy{}]$. Within the schema \eqp{ep-schematic}, the IL ensures
expectation consistency $\stackrel{\text{MM}}{\longleftrightarrow}$ in the
lower row, while the OL update equates marginals in the left column. While
this algorithm provably converges to a stationary
point of $\phi(\tvth{})$ whenever the criterion is lower bounded \cite{Opper:05}, it
is expensive to run, as variance computations $\Var_Q[\vs{}|\vy{}]$ are
required frequently during the IL optimization (convergence and properties are
discussed in the Appendix). Finally, since $\vth{} = \eta^{-1}(\tvth{} -
\vth{-})$, concave maximization w.r.t.\ $\vth{-}$ for fixed $\tvth{}$
can equivalently be seen as concave maximization w.r.t.\ $\vth{}$. We will
do the latter for notational convenience in the sequel.

\section{Scalable Variational Inference}\label{sec:scalable}

Scalable algorithms for a variational inference relaxation\footnote{
  In contrast to EP, this relaxation is convex iff all $t_i(s_i)$ are
  log-concave \cite{Nickisch:09}.}
different from EP have been proposed in \cite{Nickisch:09,Seeger:10a} (this
relaxation is called VB in the sequel, for ``Variational Bounding''). They
can be used whenever all potentials are super-Gaussian, meaning that $t_i(s_i)
= \max_{\pi_i>0} e^{b_i s_i - \frac{1}2\pi_i s_i^2 - h_i(\pi_i)/2}$ for some
$h_i(\pi_i)$, which implies the bound $-2\log Z \le \phi^{\text{VB}}(\vpi{})
:= -2\log Z_Q + h(\vpi{})$ on the log partition function of $P(\vu{}|\vy{})$
(up to an additive constant), where $h(\vpi{}) := \sum_i
h_i(\pi_i)$. Note that in this relaxation, $\vb{}$ is fixed
up front ($\vb{}=\vzero$ if all potentials $t_i(s_i)$ are even), and $\vpi{}$
are the sole variational parameters. They proceed in two steps. First,
$-2\log Z_Q = \log|\mxa{}| + \min_{\vu{*}} R(\vpi{},\vb{},\vu{*})$ (up to an
additive constant), where $R(\vpi{},\vb{},\vu{*}) :=
\sigma^{-2}\|\vy{}-\mxx{}\vu{*}\|^2 + \vs{*}^T(\diag\vpi{})\vs{*} -
2\vb{}^T\vs{*}$, $\vs{*}=\mxb{}\vu{*}$. Second, since $\vpi{}\mapsto
\log|\mxa{}|$ is a concave function, Fenchel duality
\cite[ch.~12]{Rockafellar:70} implies that
$\log|\mxa{}| = \min_{\vz{}} \vz{}^T\vpi{} - g^*(\vz{})$ for some
$g^*(\vz{})$. The variational problem becomes
\begin{eqnarray}
  && \min_{\vpi{}\succ\vzero}\phi^{\text{VB}}(\vpi{})\label{eq:sgb-problem} \\
  &=& \min_{\vz{}\succ\vzero}\min_{\vpi{}\succ\vzero,\vu{*}} \vz{}^T\vpi{} -
  g^*(\vz{}) + R(\vpi{},\vb{},\vu{*}) + h(\vpi{}). \nonumber
\end{eqnarray}
It is solved by a double loop algorithm, alternating between inner loop (IL)
minimizations w.r.t.\ $\vpi{},\vu{*}$ for fixed $\vz{}$ and outer
loop (OL) updates of $\vz{}$ and $g^*(\vz{})$.

The important difference to both the double loop algorithm of \cite{Opper:05}
and the parallel algorithm of \cite{Gerven:10} lies in the {\em decoupling}
transformation $\log|\mxa{}| = \min_{\vz{}} \vz{}^T\vpi{}
- g^*(\vz{})$. $\phi^{\text{VB}}(\vpi{})$ is hard to minimize due to the
coupling term $\log|\mxa{}|$. For example, $\nabla_{\vpi{}}\log|\mxa{}| =
\diag(\mxb{}\mxa{}^{-1}\mxb{}^T) = \Var_Q[\vs{}|\vy{}]$ requires Gaussian
variance computations, which are very expensive in practice \cite{Seeger:10a}.
But $\log|\mxa{}|$ is replaced by a fixed linear function in each IL problem,
where we can eliminate $\vpi{}$ analytically and are left with a {\em penalized
least squares} problem of the form
$\min_{\vu{*}}\sigma^{-2}\|\vy{}-\mxx{}\vu{*}\|^2 -\sum_i \psi_i(s_{* i})$, easy to
solve with standard algorithms that do not need Gaussian variances at all. To
understand the decoupling transformation more generally, consider
minimizing \eqp{sgb-problem} w.r.t.\ each variable in turn, keeping the others
fixed. The solutions are $\vu{*}=\Ex_Q[\vu{}|\vy{}]$ (means) and
$\vz{}=\nabla_{\vpi{}}\log|\mxa{}| = \Var_Q[\vs{}|\vy{}]$ (variances). The role
of decoupling is to {\em split between computations of
means and variances} \cite{Seeger:10a}: the latter, much more expensive to
obtain in general, are required at OL update points only, much less frequently
than the former (means) which are obtained by solving a single linear system.

Note that several important models come with potentials which are not
super-Gaussian (e.g., Poisson potentials for spiking neuron models
\cite{Gerwinn:08}, or potentials like the exponential, which become zero),
but can easily be handled with EP. Moreover, EP seems to be substantially more
accurate as approximation to Bayesian inference \cite{Kuss:05,Nickisch:08}.
To construct an efficient EP solver, we have to make use of
decoupling in a similar fashion, so to {\em minimize the number of Gaussian
variances computations}, while retaining provable convergence.

\section{Speeding up Expectation Propagation}\label{sec:scal-ep}

A fast and convergent EP algorithm is obtained by marrying the
double loop algorithm of \cite{Opper:05} with the decoupling trick of
\cite{Nickisch:09}. During its course, $\tvth{}$ (or $\vmu{},\vrho{}$) will
mainly be fixed, and we will drop it from notation accordingly (but recall
that the $\hat{Z}_i$ depend on it). Moreover, we will typically work with
$\vth{}=(\tvth{}-\vth{-})/\eta$ rather than $\vth{-}$. Then,
\begin{eqnarray}
  && \phi_{\cap}(\vth{}) \label{eq:phi_cap} \\
  &=& \min_{\vz{},\vu{*}} \underbrace{\vz{}^T\vpi{} - g^*(\vz{}) +
  R(\vpi{},\vb{},\vu{*}) -2\eta^{-1}\sum\nolimits_i \log\hat{Z}_i }_{=:
  \phi_{\cap}(\vv{},\vth{}),\;\; \vv{}=(\vz{},\vu{*})} \nonumber \\
  &=& \min_{\vz{},\vu{*}} \sigma^{-2}\|\vy{}-\mxx{}\vu{*}\|^2 -
  \sum\nolimits_i \psi_i(s_{* i},\pi_i,b_i) - g^*(\vz{}), \nonumber \\
  && \psi_i := -(z_i+s_{* i}^2)\pi_i + 2 b_i s_{* i} + 2\eta^{-1}\log\hat{Z}_i.
  \nonumber
\end{eqnarray}
With $\vv{}=(\vz{},\vu{*})$ and $\phi_{\cap}(\vth{}) = \min_{\vv{}}
\phi_{\cap}(\vth{},\vv{})$, the IL problem of \cite{Opper:05} is
$\max_{\vth{}}\min_{\vv{}}\phi_{\cap}$. As shown in the Appendix,
$\phi_{\cap}(\vth{},\vv{})$ is a closed proper {\em concave-convex function}
(convex in $\vv{}$ for each $\vth{}$, concave in $\vth{}$ for each $\vv{}$)
\cite{Rockafellar:70}. Strong duality holds:
$\max_{\vth{}}\min_{\vv{}}\phi_{\cap} = \min_{\vv{}}\max_{\vth{}}
\phi_{\cap}$, so the IL problem is equivalent to
\begin{eqnarray}
  &&\hspace{-.4cm} \min_{\vz{}}\left( \min_{\vu{*}} \sigma^{-2}\|\vy{}-\mxx{}\vu{*}\|^2 -
  \sum\nolimits_i \psi_i(s_{* i}) \right) - g^*(\vz{}), \nonumber \\
  &&\hspace{-.4cm} \psi_i(s_{i}) := \min_{\pi_i,b_i}\psi_i(s_{i},\pi_i,b_i). \label{eq:ow-il}
\end{eqnarray}
This problem is jointly convex in $\vz{},\vu{*}$ (note that $\psi_i(s_{* i})$ is
concave as minimum of concave functions, and the minimization over $\pi_i,b_i$
is a jointly convex problem). Solving the inner problem of \eqp{ow-il} for
fixed $\vz{}$ is a simple and very efficient penalized least squares
building block, denoted by $(\vu{*},\vth{})\leftarrow\mathtt{PLS}(\vz{},
\tvth{})$ in the sequel. Note that at its solution, $\vu{*} =
\Ex_Q[\vu{}|\vy{}]$, where $Q(\vu{}|\vy{})$ is indexed by $\vth{}$.

This means that the problem addressed in \cite{Opper:05} can be written in the
form $\min_{\vz{},\tvth{}}\phi(\vz{},\tvth{})$. The significance is the same
as in \secref{scalable}: both $\phi(\vz{},\tvth{})$ and
$\min_{\tvth{}}\phi(\vz{},\tvth{})$ (local minimum) for fixed $\vz{}$ can be
determined very efficiently. The dominating cost of computing Gaussian
variances is concentrated in the update of $\vz{}$. Two main ideas lead to
the algorithm we propose here. First, we descend on $\phi(\vz{},\tvth{})$
rather than $\phi(\tvth{}) = \min_{\vz{}}\phi(\vz{},\tvth{})$ \cite{Opper:05},
saving on variance computations. One iteration of our method determines
$\vz{}\leftarrow\Var_Q[\vs{}|\vy{}]$, then a local minimum
$\min_{\tvth{}}\phi(\vz{},\tvth{})$ in a convergent way. Empirically, such
``optimistic'' iterations seem to always descend on $\phi(\vz{},\tvth{})$
until convergence to a stationary point of $\phi(\tvth{})$, but just as for the
sequential or parallel algorithm, we cannot establish this rigorously. At 
this point, the second idea is to rely on the inner loop optimization of
\cite{Opper:05} in order to enforce descent eventually. We obtain a provably
convergent algorithm by combining optimistic steps $\min_{\tvth{}}
\phi(\vz{},\tvth{})$ for fixed $\vz{}$ with the rigorous but slow mechanism
of \cite{Opper:05}. As most, if not all optimistic steps produce sufficient
descent in practice, provable convergence comes almost for free (in
contrast to \cite{Opper:05}, where it carries a large price tag).

To flesh out this notion, denote\footnote{
  In the sequel, we will eliminate $\vu{*}$ by minimization in our notation.
  Since strong duality holds, we can move $\min_{\vu{*}}$ outside when
  solving $\mathtt{PLS}$ \eqp{ow-il} at any time (for fixed $\vz{},
  \tvth{}$).}
$\phi(\vth{},\vz{},\tvth{}) = \vz{}^T\vpi{} - g^*(\vz{}) + (\min_{\vu{*}}
R(\vpi{},\vb{},\vu{*})) - 2\eta^{-1}\sum\nolimits_i \log\hat{Z}_i$, and
$\phi(\vz{},\tvth{}) = \max_{\vth{}}\phi(\vth{},\vz{},\tvth{})$. Note that
$\phi(\vth{},\tvth{}) = \min_{\vz{}}\phi(\vth{},\vz{},\tvth{})$, moreover
$\max_{\vth{}}\phi(\vth{},\tvth{}) = \min_{\vz{}}\max_{\vth{}}\phi(\vth{},\vz{},
\tvth{})$ by strong duality. First, $\phi(\tvth{})\le \phi(\vz{},\tvth{})$, so
that $\phi(\vz{},\tvth{})$ is lower bounded if $\phi(\tvth{})$ is (which, like
\cite{Opper:05}, we assume). Next, as shown in the Appendix, we can
very efficiently minimize $\phi(\vz{},\tvth{})$ locally w.r.t.\ $\tvth{}$
by setting $\vrho{}\leftarrow\vz{}$, then iterating between
$(\vu{*},\vth{})\leftarrow\mathtt{PLS}(\vz{},\tvth{})$ and $\vmu{}\leftarrow
\vs{*}=\mxb{}\vu{*}=\Ex_Q[\vs{}|\vy{}]$. In the sequel, we denote this
subalgorithm by $\tvth{}'\leftarrow\mathtt{updateTTil}(\vz{},\tvth{})$. While
$\mathtt{updateTTil}$ may call $\mathtt{PLS}$ multiple times, it does not
require expensive Gaussian variance computations. An ``optimistic'' step of our
algorithm updates $\vz{}'\leftarrow\Var_Q[\vs{}|\vy{}]$, then
$\tvth{}'\leftarrow\mathtt{updateTTil}(\vz{}',\tvth{})$, at the cost of one
variance computation. Within the schema \eqp{ep-schematic}, we update $\vz{}$,
set $\vrho{}\leftarrow\vz{}$, then attain expectation consistency and
$\vmu{}\stackrel{!}{=}\Ex_Q[\vs{}|\vy{}]=\vs{*}=\mxb{}\vu{*}$ for fixed
variances $\vz{}$, $\vrho{}$.

Suppose we are at a point $\vz{}, \tvth{}$ (and $\vth{}$), so that $\tvth{}$ is
a local minimum point of $\phi(\vz{},\tvth{})$. How can we descend:
$\phi(\vz{}',\tvth{}') < \phi(\vz{},\tvth{})$ unless $\tvth{}$ is a stationary
point of $\phi(\tvth{})$? Let $\vth[1]{}=\vth{}$. The optimistic step would be
$\vz[1]{}=\Var_Q[\vs{}|\vy{}]$, then
$\tvth{}'\leftarrow\mathtt{updateTTil}(\vz[1]{},\tvth{})$. If
$\phi(\vz[1]{},\tvth{}')$ is sufficiently smaller than $\phi(\vz{},\tvth{})$,
we are done with our descent step: $\vz{}'=\vz[1]{}$. Otherwise, we
run one iteration $\vth[1]{}\to\vth[2]{}$ of the inner optimization
$\max_{\vth{}}\phi(\vth{},\tvth{})$
of \cite{Opper:05}. This requires variance
computations, while $\vz[1]{}$ can be reused (and $\vz[2]{}$ may already be
computed). We set $\vth{}\leftarrow\vth[2]{}$ and attempt another optimistic
step: $\vz[2]{}$, $\mathtt{updateTTil}(\vz[2]{},\tvth{})$. Without intervening
descent, we would eventually obtain
$\vth[k]{}=\max_{\vth{}}\phi(\vth{},\tvth{})$, thus $\vz[k]{} = \argmin_{\vz{}'}
\phi(\vz{}',\tvth{})$. If no descent happens from there, $\tvth{}$ must be
a stationary point of $\phi(\tvth{})$ (see \cite{Opper:05} and Appendix).

Note that in most cases in practice, our algorithm does not run into the
inner optimization of \cite{Opper:05} even once (unless to confirm final
convergence). Yet the possibility of doing so is what makes our convergence
proof work. \algref{main-alg} provides a schema.

\begin{algorithm}[ht]
\begin{algorithmic}
\STATE{$\Delta(a,b) := (b-a)/\max\{|a|,|b|,10^{-9}\}$. \\
  Iterate over $\vz{}, \tvth{}\leftrightarrow(\vmu{},\vrho{})$.}
\REPEAT
  \STATE{$\vpi[1]{}=\vpi{}$.}
  \FOR{$k=1,2,\dots$}
    \STATE{$\vz[k]{}=\Var_Q[\vs{}|\vy{}]$.}
    \STATE{$(\tvth{}',\vth{}')\leftarrow\mathtt{updateTTil}(\vz[k]{},\tvth{})$.}
    \IF{$\Delta(\phi(\vz[k]{},\tvth{}'),\phi(\vz{},\tvth{})) > \eps$}
      \STATE{Sufficient descent: $\vz{}\leftarrow\vz[k]{}$, $\tvth{}\leftarrow
        \tvth{}'$, \\
        $\vth{}\leftarrow\vth{}'$. Leave loop over $k$.}
    \ELSE
      {\color{gray}
      \STATE{Run iteration of $\max_{\vth{}}\phi(\vth{},\tvth{})$: \\
        $\vth[k]{}\to\vth[k+1]{}$. Set $\vth{}\leftarrow\vth[k+1]{}$.}
      \IF{$|\Delta(\phi(\vth[k+1]{},\tvth{}),\phi(\vth[k]{},\tvth{}))| < \eps$}
        \STATE{Converged to stationary point $\tvth{}$: Terminate algorithm.}
      \ENDIF
      }
    \ENDIF
  \ENDFOR
\UNTIL{Maximum number of iterations done}
\end{algorithmic}
\caption{\label{alg:main-alg} Double loop EP algorithm. \newline
  The part shaded in grey was never accessed in our experiments (see text for
  comments).
}
\end{algorithm}

A word of warning about the inner optimization
$\max_{\vth{}}\phi(\vth{},\tvth{})$. From \eqp{ow-il}, it is
tempting to iterate between $\vz{}\leftarrow\Var_Q[\vs{}|\vy{}]$ and
$(\vu{*},\vth{})\leftarrow\mathtt{PLS}(\vz{},\tvth{})$. However, this
does not lead to descent and typically fails in practice. As seen in
\secref{scalable}, the update of $\vz{}$ serves to refit an {\em upper} bound,
suitable for {\em minimizing}, but not {\em maximizing} over $\vth{}$. In our
algorithm, this problem is compensated by the minimization over $\tvth{}$:
optimistic steps seem to always descend.

\subsection{Computational Details}\label{sec:inner-opt}

In this section, we provide details for computational primitives required
in \algref{main-alg}. First, we show how to efficiently compute
$\mathtt{PLS}$, i.e.\ solve the inner problem in \eqp{ow-il} for fixed
$\vz{}\succ\vzero$. As all $\psi_i(s_{* i})$ are concave, this is a convex
penalized least squares problem, for which many very efficient solvers are
available. A slight technical challenge comes
from the implicit definition of the regularizer: evaluating $\psi_i$ and
its derivatives entails a bivariate convex minimization.

In our experiments, we employ a standard gradient-based Quasi-Newton optimizer.
Suppose we are at $\vu{*}$ and have determined the maximizer $\vth{} =
(\vpi{},\vb{})$. If
$f(\vu{*}) = \sigma^{-2}\|\vy{}-\mxx{}\vu{*}\|^2 - \sum_i\psi_i(s_{* i})$, then
$\psi_i'(s_{* i}) = \partial_{s_{* i}}\psi_i(s_{* i},\pi_i,b_i) = 2(b_i - \pi_i s_{* i})$,
so that $\nabla_{\vu{*}}f(\vu{*}) = 2\sigma^{-2}\mxx{}^T(\mxx{}\vu{*}-\vy{}) +
2\mxb{}^T(\vpi{}\circ\vs{*}-\vb{})$, at the cost of one matrix-vector
multiplication (MVM) with $\mxx{}^T\mxx{}$, $\mxb{}^T$, $\mxb{}$ respectively
(here, ``$\circ$'' denotes the componentwise product). For the bivariate
minimizations, the derivatives are $\partial_{b_i}\psi_i =
2(s_{* i} - \Ex_{\hat{P}_i}[s_i])$, $\partial_{\pi_i}\psi_i = -(z_i+s_{* i}^2) +
\Ex_{\hat{P}_i}[s_i^2]$: we have to adjust $b_i,\pi_i$ so that mean and
variance of $\hat{P}_i$ coincides with $s_{* i}$ and $z_i$. Details for the
computation of $\hat{P}_i$ are given in \cite{Seeger:07d}. In our
implementation, we initialize the minimization by two standard EP updates,
then run Newton's algorithm (details are given in a longer paper). Even for
large $q$, these bivariate minimizations can often be done more rapidly than
MVMs with $\mxx{}^T\mxx{}$. Moreover, they can be solved
in parallel on graphics hardware.

The inner optimization $\max_{\vth{}}\phi(\vth{},\tvth{})$ of \cite{Opper:05}
can be addressed by any convex solver. We employ Quasi-Newton once more. The
gradients are $\partial_{\vb{}}\phi(\vpi{},\vb{},\tvth{}) = 2(
(\Ex_{\hat{P}_i}[s_i])-(\Ex_Q[s_i]) )$,
$\partial_{\vpi{}}\phi(\vpi{},\vb{},\tvth{}) = (\Ex_Q[s_i^2]) -
(\Ex_{\hat{P}_i}[s_i^2])$. This computation entails $\vz{}=\Var_Q[\vs{}|\vy{}]$.
Note
that with a standard solver, a sufficient increase in $\phi(\vth{},\tvth{})$
(for fixed $\tvth{}$) may require a number of $\Var_Q[\vs{}|\vy{}]$
computations. We are not aware of an effective way to decouple this problem
as in \secref{scalable}.

\subsubsection*{Gaussian Variances}

Finally, how do we compute Gaussian variances $\vz{}=\Var_Q[\vs{}|\vy{}] =
\diag(\mxb{}\mxa{}^{-1}\mxb{}^T)$? This is by far the most expensive
computation in all EP algorithms discussed here: our main contribution is a
novel convergent algorithm which requires few of these calls. In our
experiments, $n$ is a few thousand, $q\approx 3 n$, and we can maintain an
$n\times n$ matrix in memory. We use the identity
\[
  \vz{}=\diag\left( \mxb{}\mxa{}^{-1}\sum\nolimits_i\vdelta{i}
  \vdelta{i}^T\mxb{} \right) = \sum\nolimits_i
  (\mxb{}\mxa{}^{-1}\vdelta{i})\circ(\mxb{}\vdelta{i}),
\]
where $\vdelta{i} = (\Ind{j=i})_j$.
We compute the Cholesky decomposition $\mxa{}=\mxl{}\mxl{}^T$, then
$\mxa{}^{-1}$ from $\mxl{}$, using LAPACK code, then accumulate
$\vz{}$ by $2 n$ MVMs with $\mxb{}$.

\begin{figure}[t]
\begin{centering}
\includegraphics[width=1\columnwidth]{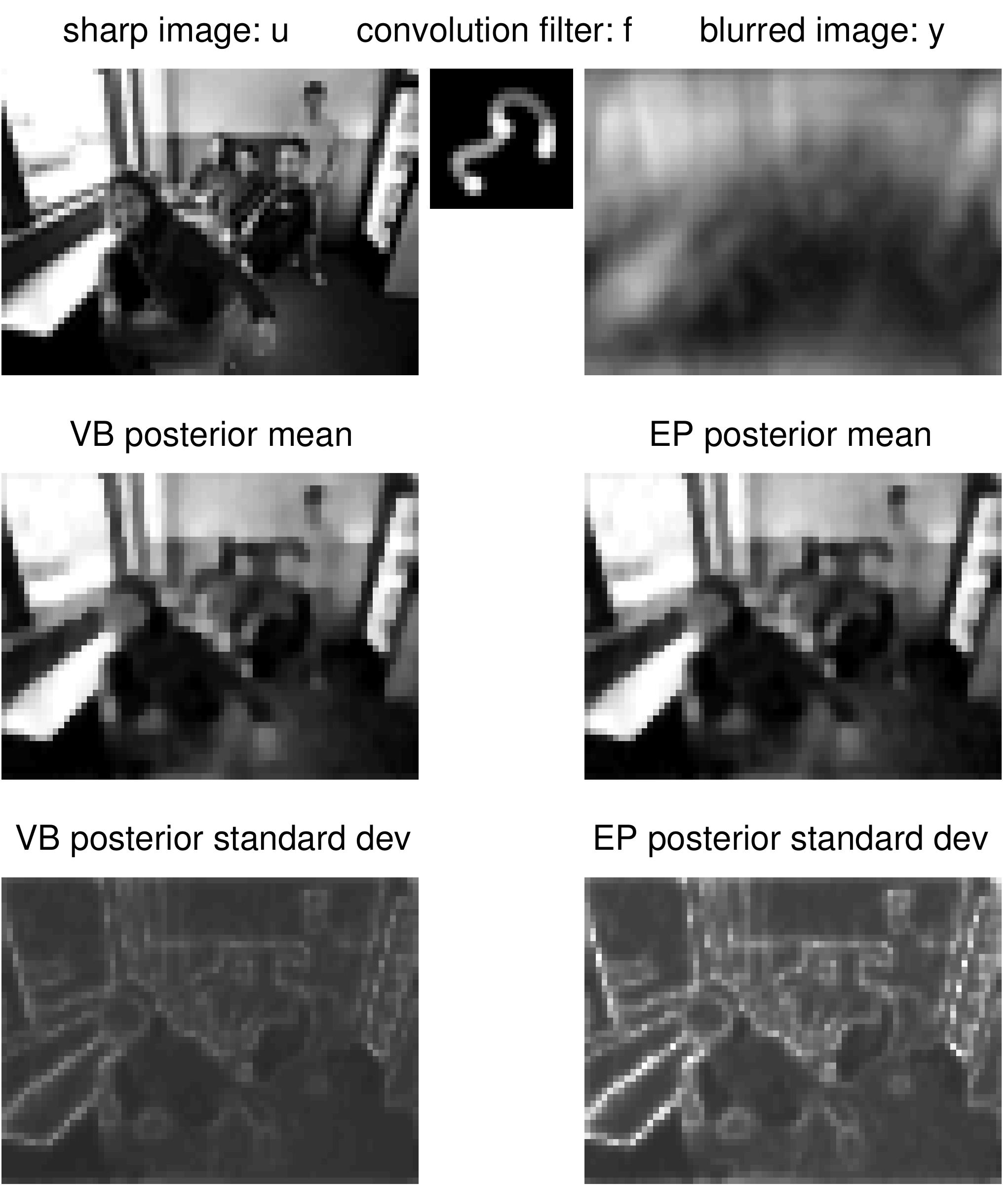}
\par\end{centering}
\caption{\label{fig:deconv_margs} Deconvolution setting and resulting
  marginals (variances on $\vu{}$, not on $\vs{}$). $\vu{}$ is $48\times 73$,
  pixels, the kernel $\vf{}$ is $22\times 25$ ($n=3504$, $q=10512$,
  $\tau_a=\tau_r=15$, $\sigma^2=10^{-5}$).
}
\end{figure}

If $n$ is larger than $10^4$ or so, this approach is not workable anymore.
If $\mxa{}$ is very sparse, it may possess a sparse Cholesky decomposition
which can be determined efficiently, in which case $\vz{}$ is determined
easily \cite{Gerven:10}. However, for typical image reconstruction models,
$\mxx{}$ is dense. For the VB relaxation of \secref{scalable},
variances have been approximated by the Lanczos algorithm \cite{Seeger:08a,
Nickisch:09}. It is noted in \cite{Seeger:10a} that variances are strongly
(but selectively) underestimated in this way, and consequences for the VB
double loop algorithm are established there: in a nutshell, while
outcomes are qualitatively different, the algorithm behaviour remains
reasonable. In contrast, if any of the EP algorithms discussed in this paper
are run with Lanczos variance approximations, they exhibit highly erratic
behaviour. Parallel EP \cite{Gerven:10} rapidly diverges, our variant
ends in numerical breakdown. While we are lacking a complete explanation for
these failures at present, it seems evident that the expectation consistency
conditions, whose structure is more complicated than the simple VB bound, do
not tolerate strong variance errors. Our observation underlines the thesis of
\cite{Seeger:10a}. Robustness to variance errors of the kind produced by
Lanczos becomes an important asset of variational inference relaxations, at
least if large scale inference is to be addressed. The EP relaxation, as it
stands, does not seem to be robust in this sense. Explaining this fact, and
possibly finding a {\em robust modification} of the expectation consistency
conditions, remain important topics for future research.

\section{Experiments}\label{sec:exper}

\subsection{Expectation Propagation vs.\ VB}\label{sec:exp-epvb}

In the following experiment, we compare approximate inference outcomes of EP
(\secref{expprop}) and VB (\secref{scalable}), complementing previous
studies \cite{Kuss:05,Nickisch:08}. We address the (non-blind)
deconvolution problem for image deblurring (details ommitted here are found in
\cite{Levin:09}): $\vu{}\in\R^n$ represent the desired sharp image, $\mxx{} =
(\diag\tvf{})\mxf{n}$, where $\mxf{n}$ is the $n\times n$ discrete Fourier
transform (DFT)\footnote{
  Strictly speaking, we encode $\C$ by $\R^2$, and $\mxf{n}$ is the
  ``real-to-complex'' DFT (closely related to the discrete cosine transform).
  Both $\tvf{}$ and $\tvy{}$ are Hermitian and can be stored as $\R^n$
  vectors.},
$\tvf{}=\mxf{n}\vf{}$ the spectrum of the blur kernel $\vf{}$, and
$\vy{}=\mxf{n}\tvy{}$, $\tvy{}$ the blurry image. Our model setup is similar
to what was previously
used in \cite{Seeger:10a}: $P(\vu{})$ is a Laplace sparsity prior (see
\secref{intro}), the transform $\mxb{}$ consists of an orthonormal wavelet
transform $\mxb{a}$ and horizontal/vertical differences $\mxb{r}$
(``total variation''), corresponding prior parameters are $\tau_a$, $\tau_r$.
Recall that $\vb{}$ is fixed\footnote{
  This is an inherent feature of the variational bound, which would cease to
  be valid if $\vb{}$ were optimized over.}
depending on the $t_i(\cdot)$ in VB: since they are
even, $\vb{}=\vzero$. In contrast, they are free variational parameters in
EP. Posterior marginals, as approximated by EP and VB, are shown
in \figref{deconv_margs}, while we compare parameters $\vb{}$, $\vpi{}$ in
\figref{deconv_params}.

The EP and VB approximations are substantially different. While the means
are visually similar, EP's posterior variances are larger and show a more
pronounced structure. An explanation is offered by the striking differences
in final parameters $\vb{}$, $\vpi{}$. Roughly, $\pi_i$ scales the degree of
penalization of $s_i$ \cite{Seeger:07d}. While both EP and VB strongly
penalize certain coefficients, VB (in contrast to EP) seems to universally
penalize all $s_i$ (all $\pi_{\text{VB},i}>10$), thus may produce small
variances simply by overpenalization. EP clearly makes use of $\vb{}$, which
allow to control the posterior mean independent of the
covariance: a mechanism not available for VB. It is important to note that
our findings are in line with those in \cite{Nickisch:08}, who found that
VB strongly underapproximated marginal variances (they obtained the ground
truth by expensive Monte Carlo simulations). As noted in \secref{intro}, it is
often the posterior uncertainty estimates (covariances) which give Bayesian
decision-making an edge over point estimation approaches.

\subsection{EP Timing Comparison}
\label{sec:exp-timing}

In this section, we provide timing comparisons between EP algorithms discussed
in this paper. Our setup is much the same as in \secref{exp-epvb}, but both
the choice of $\mxx{}$ and data is taken from
\cite{Seeger:10a}. The problem is inference over images $\vu{}\in\R^n$ from
``Cartesian MRI'' measurements (discrete Fourier coefficients) $\vy{}\in\C^m$,
so
that $\mxx{}=\Id_{J,\cdot}\mxf{n}$, where $J$ is an index selecting acquired
coefficients (in fact, complete columns in DF space (``phase encodes'') are
sampled, according to a design optimized for natural images). The prior is the
same as used above.

\begin{figure}[h]
\begin{centering}
\includegraphics[width=1\columnwidth]{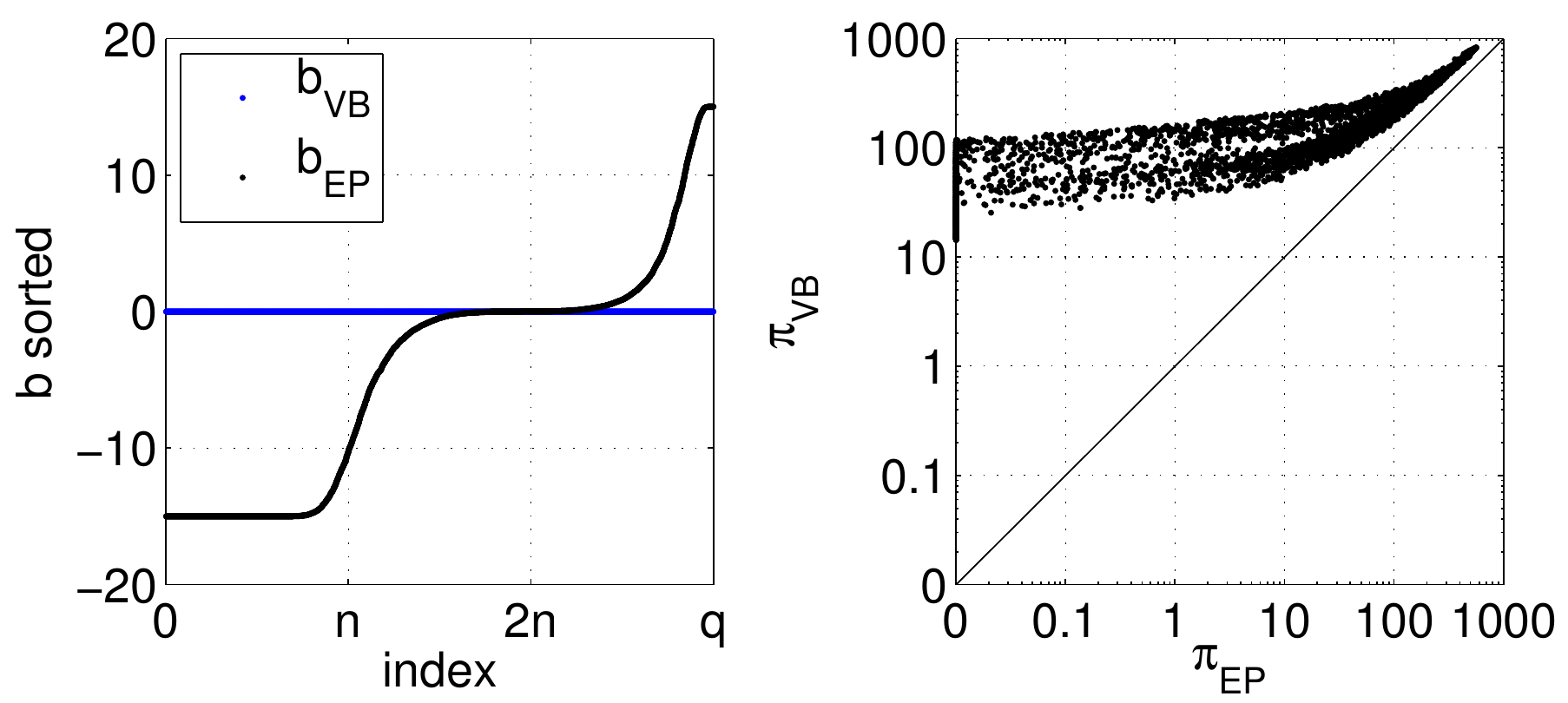}
\par\end{centering}
\caption{\label{fig:deconv_params} Final parameters for
  deconvolution. Left: $\vb{}$ sorted ($\vb{\text{VB}}=\vzero$ by
  construction). Right: $\vpi{}$.
}
\end{figure}

\begin{figure*}[t]
\begin{centering}
\includegraphics[width=1\columnwidth]{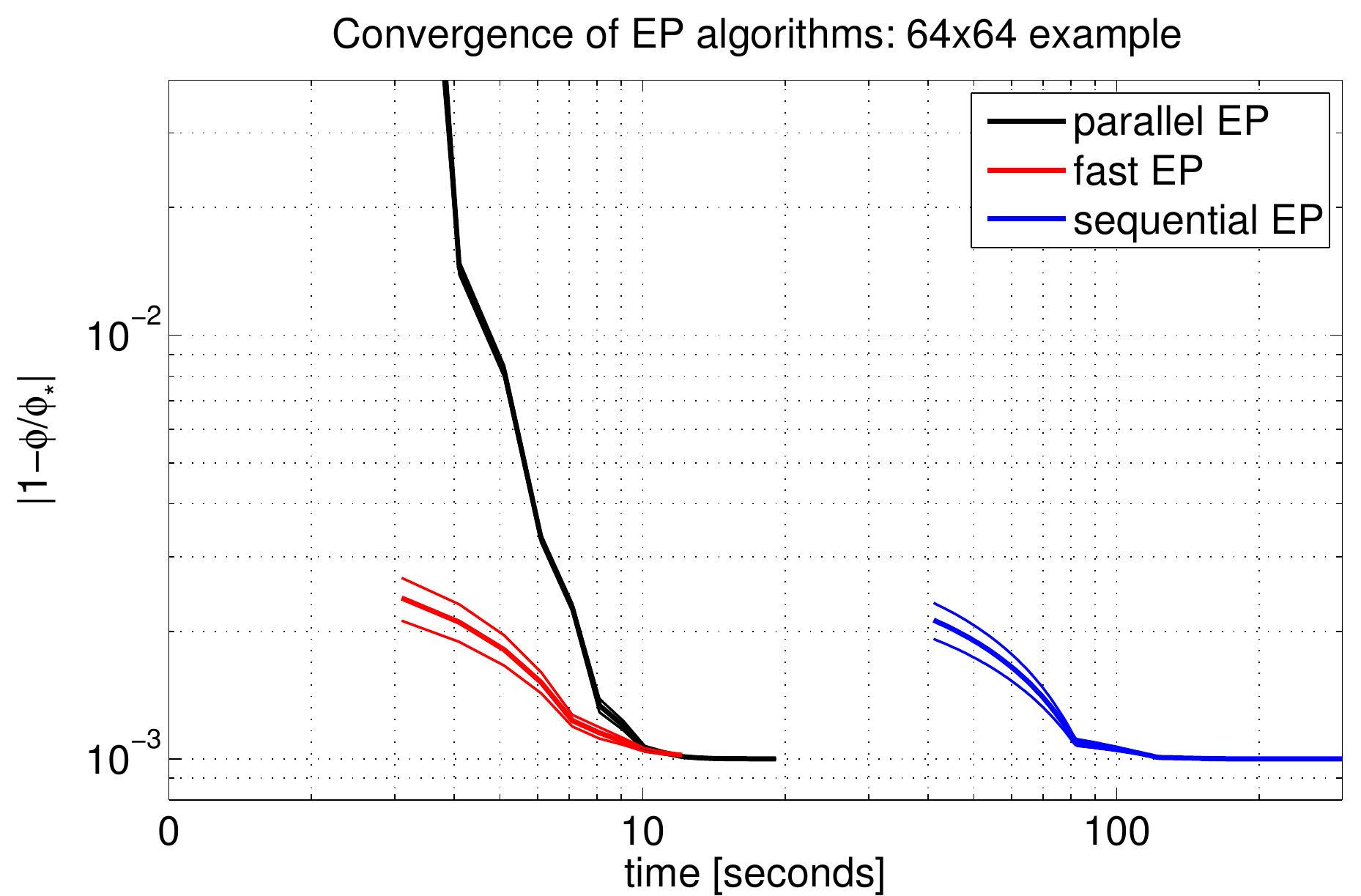}\hfill
\includegraphics[width=1\columnwidth]{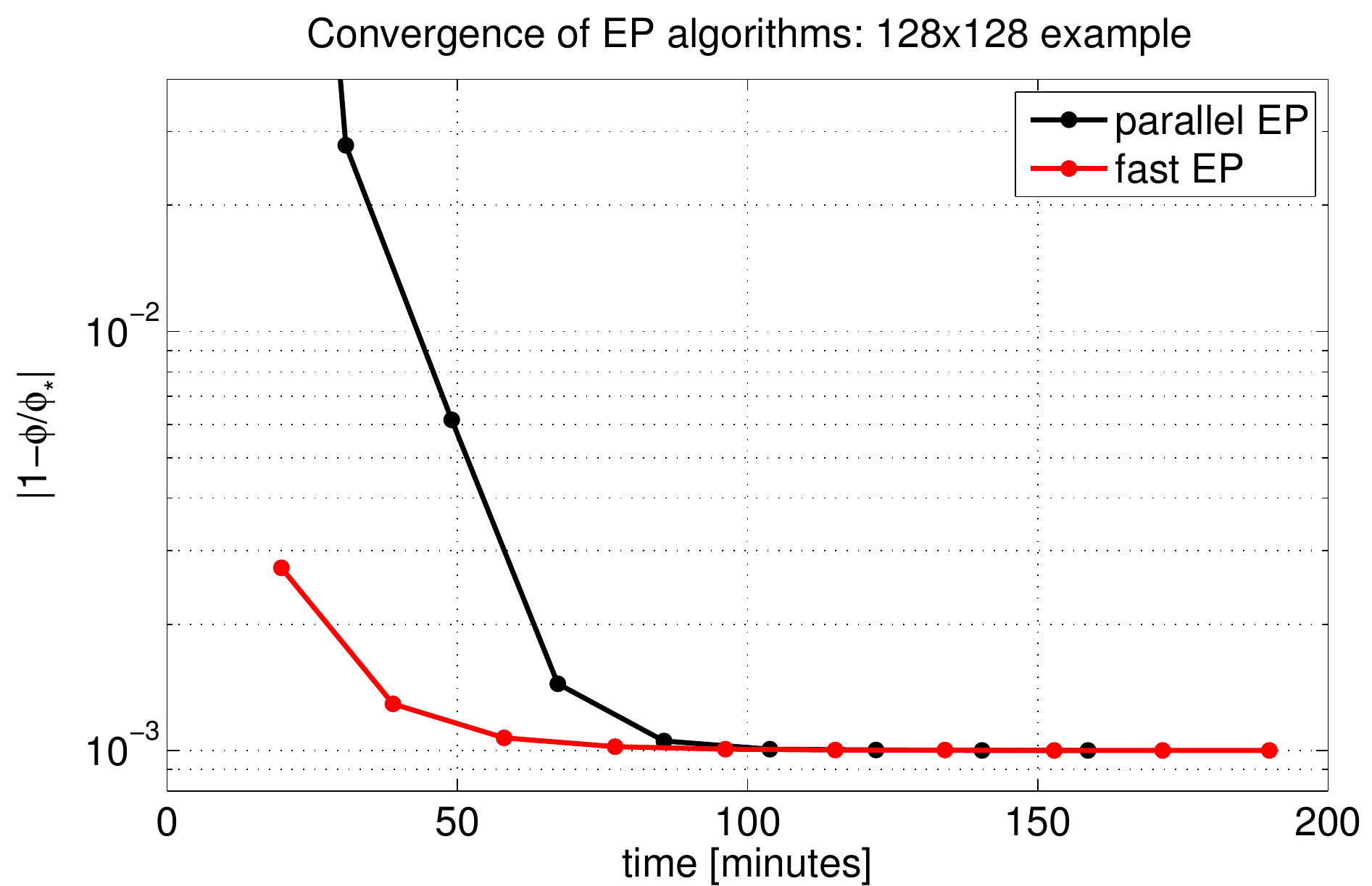}
\par\end{centering}
\caption{\label{fig:results} Timing comparison of EP algorithms for inference
  over greyscale images. Left: $64\times 64$ images. Right: $128\times 128$
  image. Shown is relative distance to EP energy stationary point
  $|(\phi-\phi_*)/\phi_*|$ as function of running time (left: mean, two
  std.\ over 20 different images). \newline
  Algorithms: sequential
  EP (\secref{expprop}; left only), parallel EP (\secref{expprop}), and fast EP
  (our method).
}
\end{figure*}

In our first experiment, we use $64\times 64$ images ($n=4096$, $q=12160$) and
a design $\mxx{}$ sampling 16 columns ($m=1024$, 4 times undersampled). We
compare the sequential and parallel EP algorithms with our novel fast
(convergent) EP method. We chose not to include results for the double loop
algorithm of \cite{Opper:05}, since it runs even slower than the sequential
method (see comments in \secref{inner-opt}). Our results are averaged over
20 different images (the $\vy{}$ vectors are noisy acquisitions, $\sigma^2 =
10^{-3}$, but the same across methods). Moreover, $\tau_a=0.04/\sigma$,
$\tau_r=0.08/\sigma$ (same values as in \cite{Seeger:10a}). Timing runs were
done on an otherwise unloaded standard desktop machine. For each run, we
stored tupels $(T_j,\phi_j)$ at the end of each outer iteration (for
sequential EP, this is a sweep over all potentials), $T_j$ elapsed
time (in secs), $\phi_j$ the EP energy value attained. On a fixed image, all
methods eventually attained the same energy value\footnote{
  While this is not guaranteed by present EP convergence theory, it happened
  in all our cases.}
(say, $\phi_*$), and we show $(T_j,|(\phi_j-\phi_*)/\phi_*|)$. Results are
presented in \figref{results}, left.
First, the sequential algorithm is not competitive with the others. At a time
when the others converged, it is roughly $1/4$ through its first sweep (while
requiring about four sweeps to converge). Second, the parallel and our fast
EP algorithm converge in about the same time. However, ours does so much more
smoothly and attains a near optimal solution more quickly.

In a second experiment, we use a single $128\times 128$ image
($n=16384$, $q=48896$) and a design $\mxx{}$ sampling 36 columns ($\approx
3.5$ times undersampled). We compare the parallel with our fast EP algorithm,
since the sequential method is clearly infeasible at this scale. Here,
$\sigma^2 = 2\cdot 10^{-4}$, $\tau_a=0.04/\sigma$, $\tau_r=0.08/\sigma$.
Results are presented in \figref{results}, right. On this larger problem,
our algorithm converges significantly faster.

Our method (fast EP in \figref{results}) is provably convergent, while
parallel EP (and sequential EP) lacks such a guarantee. Beyond, the main
difference between fast and parallel EP lies in how thoroughly variance
computations are exploited. Fast EP spends more effort between them,
solving $\min_{\tvth{}}\phi(\vz{},\tvth{}) = \min_{\tvth{}}\max_{\vth{}}
\phi(\vth{},\vz{},\tvth{})$, while parallel EP simply does a single EP
update. Our method therefore incurs an overhead, which motivates the results
for $64\times 64$ images. However, this overhead is modest (each step of
$\mathtt{PLS}$ costs $O(q + n\log n)$), while the cost for variances, at
$O(n (n^2 + q))$, grows very fast. The overhead for fast EP pays off in the
$128\times 128$ image example, due to the fact that it requires about two
variance computations less than parallel EP to attain convergence. Notably,
the overhead cost can still be greatly reduced by
running different algorithms (see \secref{discuss}) or parallelizing the
computations of the $\psi_i(s_{* i})$, which is not done in our implementation.

\section{Discussion}\label{sec:discuss}

We proposed a novel, provably convergent algorithm to solve the expectation
propagation relaxation of Bayesian inference. Based on the insight that the
most expensive computations by far in any variational method concern Gaussian
variances, we exploit a decoupling trick previously used in
\cite{Wipf:08,Nickisch:09} in order to minimize the number of such
computations. Our method is at least an order of magnitude faster than the
commonly used sequential EP algorithm, and improves on parallel EP
\cite{Gerven:10}, the previously fastest solver we are aware of, both in
running time and guaranteed convergence.
Moreover, it is in large parts similar to recent algorithms for other
relaxations \cite{Nickisch:09}, which allows for transfer of efficient code.
While the sequential EP algorithm is most widely used today, our results
indicate that this is wasteful even for small and medium size problems and
should be avoided in the future.

There are numerous avenues for future work. First, for problems of the
general form discussed in \secref{exper}, the central penalized least
squares primitive $\mathtt{PLS}$ could be solved more efficiently by employing
modern augmented Lagrangian techniques, such as the ADMM algorithm reviewed
in \cite{Combettes:10} (today's most efficient sparse deconvolution
algorithms are based on this technique), and by parallelizing the innermost
bivariate optimization problems leading to $\psi_i(s_{* i})$ and its derivatives.
Such measures would bring down the (already modest) overhead of our technique,
compared to parallel EP. Moreover, we aim to resolve whether the
``optimistic steps'' our algorithm is mainly based on, provably lead to
descent by themselves (this would render the fallback on \cite{Opper:05},
shaded in \algref{main-alg}, obsolete, thus simplify the code).

Known EP algorithms (including ours presented here) break
down in the presence of substantial Gaussian variance approximation errors,
in contrast to algorithms for simpler relaxations which behave robustly. If
real-world Bayesian image applications such as those in \secref{exper} are to
be run at realistic sizes, variance errors cannot be avoided. The most
important future direction is therefore to understand the reason for this
non-robustness of EP algorithms (or even the expectation-consistency conditions
as such) and to seek for alternatives which combine the accuracy of this
relaxation with good behaviour in the presence of typical Gaussian variances
approximation errors \cite{Seeger:10a}.

\subsection*{Appendix}

We start by reviewing the convergence proof for the EP double loop algorithm
of \secref{expprop} \cite{Opper:05}. The problem is $\min_{\tvth{}}
\max_{\vth{-}} \phi_{\cap}(\vth{-},\tvth{}) + \phi_{\cup}(\tvth{})$. Now,
$\phi_{\cap}(\tvth{}) = \max_{\vth{-}} \phi_{\cap}(\vth{-},\tvth{})$ is concave.
If $\vth{-}=\argmin\phi(\vth{-},\tvth{})$, then $\phi(\tvth{}')\le R(\tvth{}')
:= \phi_{\cap}(\vth{-},\tvth{}) - \vg{}^T(\tvth{}'-\tvth{}) +
\phi_{\cup}(\tvth{}')$, where $\vg{} = -\nabla_{\tvth{}}\phi_{\cap}(\tvth{}) =
-\partial_{\tvth{}}\phi_{\cap}(\vth{-},\tvth{})$ \cite[ch.~12]{Rockafellar:70}.
If $\vth{} = \eta^{-1}(\tvth{}-\vth{-})$, then $\vg{} = \partial_{\tvth{}}
2\log Z_Q = \eta^{-1}(\Ex_Q[\vs{}|\vy{}], -\frac{1}2\Ex_Q[\vs{}^2|\vy{}])$. Now,
$\phi_{\cup}(\tvth{})=R(\tvth{})$, and $R(\tvth{}')$ is convex, its minimum
defined by $\nabla_{\tvth{}'}\phi_{\cup}(\tvth{}') = \vg{}$. Therefore,
minimizing $R(\tvth{}')$ leads to $\phi(\tvth{}')<\phi(\tvth{})$, unless
$\vg{} = \nabla_{\tvth{}}\phi_{\cup}(\tvth{})$, thus
$\nabla_{\tvth{}}\phi(\tvth{}) = \vzero$. Since the sequence $\phi(\tvth{})$
is nonincreasing and lower bounded, it must converge to a stationary point.
To determine $\vg{}$, note that if $\vu{*}$ is the minimizer in \eqp{ow-il},
then $\Ex_Q[\vs{}|\vy{}]=\vs{*}=\mxb{}\vu{*}$ and
$\Ex_Q[\vs{}^2|\vy{}] = \vs{*}^2 + \Var_Q[\vs{}|\vy{}]$. Moreover, since
$\phi(\tvth{}')$ is the sum of log partition functions of $N(\mu_i,\rho_i)$,
the equation $\nabla_{\tvth{}'}\phi_{\cup}(\tvth{}') = \vg{}$ is solved by
$\vmu{}'=\vs{*}$, $\vrho{}'=\Var_Q[\vs{}|\vy{}]$.

Importantly, exactly the same argument establishes the convergence (to a
stationary point) of $\min_{\tvth{}'}\phi(\vz{},\tvth{}')$ for any fixed
$\vz{}\succ\vzero$, thus the computation of $\mathtt{updateTTil}$ in
\secref{scal-ep}. We only have to replace $\log|\mxa{}(\vpi{})|$ by
$\vz{}^T\vpi{} - g^*(\vz{})$ (both are concave in $\vth{}$, therefore concave
in $(\vth{-},\tvth{})$), noting that the gradient w.r.t.\ $\vpi{}$ changes
from $\nabla_{\vpi{}}\log|\mxa{}| = \Var_Q[\vs{}|\vy{}]$ to $\nabla_{\vpi{}}
(\vz{}^T\vpi{} - g^*(\vz{})) = \vz{}$. The only difference to the algorithm
of \cite{Opper:05} just discussed is that $\vrho{}$ is updated to $\vz{}$,
not to $\Var_Q[\vs{}|\vy{}]$, so that variances do not have to be computed.

Next, we establish the properties of the inner loop problem
$\max_{\vth{}}\phi_{\cap}(\vth{},\tvth{})$
(Eqs.~\ref{eq:phi_cap}, \ref{eq:ow-il}).
In particular, we prove that strong duality holds. Recall that
$\vv{}=(\vz{},\vu{*})$ and
$\phi_{\cap}(\vv{},\vth{})$ from \eqp{phi_cap}. We begin by extending
$\phi_{\cap}(\vv{},\vth{})$ for all values of $\vz{}$ and $\vpi{}$
\cite{Rockafellar:70}. First, $g^*(\vz{}) = \inf_{\vpi{}}\vz{}^T\vpi{} -
\log|\mxa{}(\vpi{})|$ is the {\em concave} dual function of
$\log|\mxa{}(\vpi{})|$. Since $\log|\mxa{}|\to\infty$ whenever any
$\pi_i\to\infty$ \cite{Seeger:08d}, then $g^*(\vz{})\to -\infty$ as any
$z_i\searrow 0$, and $\phi_{\cap} := +\infty$ if any $z_i\le 0$. Moreover,
$\phi_{\cap} := -\infty$ if $\vz{}\succ\vzero$ and any $\pi_i<0$, and
$\phi_{\cap}(\vv{},\vpi{},\vb{}) := \lim_{\tvpi{}\searrow\vpi{}}
\phi_{\cap}(\vv{},\tvpi{},\vb{})$ for any $\vpi{}\succeq\vzero$. With these
extensions, it is easy to see that $\phi_{\cap}(\vv{},\vth{})$ is a closed
proper concave-convex function \cite[ch.~33]{Rockafellar:70}: convex in
$\vv{}$ for each $\vth{}$, concave in $\vth{}$ for each $\vv{}$.
Note that we always have that $\max_{\vth{}}\min_{\vv{}}\phi_{\cap}\le
\min_{\vv{}}\max_{\vth{}}\phi_{\cap}$ (weak duality). In order to establish
equality (strong duality), we show that $\phi_{\cap}(\cdot,\vth{})$ do not have
a common nonzero direction of recession.
Given that, strong duality follows from \cite[Theorem~37.3]{Rockafellar:70}.

\begin{theo}
Let $\phi(\vv{},\vth{})$ be defined as in \eqp{phi_cap}, and extended to a
closed proper concave-convex function. If $\vth{}=(\vpi{},\vb{})$ is such that
$\vpi{}\succ\vzero$ and $\mxa{}(\vpi{})$ is positive definite, then
$\phi(\cdot,\vth{})$ has no nonzero direction of recession. For any
$\vd{}\ne\vzero$ and any $\vv{}$ so that $\phi(\vv{},\vth{})<\infty$:
\[
  \lim_{t\to\infty} \frac{\phi(\vv{}+t\vd{},\vth{})-\phi(\vv{},\vth{})}t > 0.
\]
\end{theo}
\begin{proof}
Write $F(\vv{}) = \phi_{\cap}(\vv{},\vth{})$ for brevity, and pick any
$\vd{}\ne\vzero$. $\vd{}$ is a direction of recession iff $\lim_{t\to\infty}
(F(\vv{}+t\vd{})-F(\vv{}))/t\le 0$ for some $\vv{}$
\cite[Theorem~8.5]{Rockafellar:70}. Pick any $\vv{}=(\vz{},\vu{*})$,
$\vz{}\succ\vzero$, and let $\vd{}=(\vd{z},\vd{u})$. If $\vd{u}\ne\vzero$, then
$F(\vv{}+t\vd{}) = \Omega(t^2)$ by the positive definite quadratic
part. If $(\vd{z})_i<0$ for any $i$, then there is some $t_0>0$ so that
$(\vz{}+t\vd{z})_i$ is negative and $F(\vv{}+t\vd{})=\infty$ for all $t\ge t_0$.
This leaves us with $\vd{u}=\vzero$,
$\vd{z}\succeq\vzero$, so that $(\vd{z})_i>0$ for some $i$. Let
$\tvpi{} = \vpi{} - (\pi_i/2)\vdelta{i}$.
By definition, $g^*(\vz{}+t\vd{z})\le (\vz{}+t\vd{z})^T\tvpi{} -
\log|\mxa{}(\tvpi{})|$, therefore
\[
\begin{split}
  & \frac{F(\vv{}+t\vd{}) - F(\vv{})}{t} = \vd{z}^T\vpi{} +
  \frac{g^*(\vz{}) - g^*(\vz{}+t\vd{z})}{t} \\
  \ge\, & \vd{z}^T(\vpi{}-\tvpi{}) +
  \frac{g^*(\vz{}) + \log|\mxa{}(\tvpi{})| - \vz{}^T\tvpi{}}{t} \\
  =\, & \pi_i(\vd{z})_i/2 + \frac{g^*(\vz{}) + \log|\mxa{}(\tvpi{})| -
  \vz{}^T\tvpi{}}{t},
\end{split}
\]
which is positive as $t\to\infty$.
\end{proof}



\end{document}

%% file: mymacros.tex
\newcommand{\field}[1]{\mathbb{#1}}

\newcommand{\R}{\field{R}}
\newcommand{\C}{\field{C}}

\newcommand{\vect}[1]{\boldsymbol{#1}} 
\newcommand{\mat}[1]{\boldsymbol{#1}} 
\newcommand{\tvect}[1]{\tilde{\boldsymbol{#1}}} 
\newcommand{\tmat}[1]{\tilde{\boldsymbol{#1}}} 
\newcommand{\hvect}[1]{\hat{\boldsymbol{#1}}} 
\newcommand{\hmat}[1]{\hat{\boldsymbol{#1}}} 
\newcommand{\vzero}{\vect{0}}

\newcommand{\dummystring}{QWERTYU}
\newcommand{\vci}[3][\dummystr]{\ifthenelse{\equal{#1}{\dummystring}}{\vect{#2}_{#3}}{\vect{#2}_{#3}^{(#1)}}}
\newcommand{\mx}[3][\dummystr]{\ifthenelse{\equal{#1}{\dummystring}}{\mat{#2}_{#3}}{\mat{#2}_{#3}^{(#1)}}}
\newcommand{\tvci}[3][\dummystr]{\ifthenelse{\equal{#1}{\dummystring}}{\tvect{#2}_{#3}}{\tvect{#2}_{#3}^{(#1)}}}
\newcommand{\tmx}[3][\dummystr]{\ifthenelse{\equal{#1}{\dummystring}}{\tmat{#2}_{#3}}{\tmat{#2}_{#3}^{(#1)}}}
\newcommand{\hvci}[3][\dummystr]{\ifthenelse{\equal{#1}{\dummystring}}{\hvect{#2}_{#3}}{\hvect{#2}_{#3}^{(#1)}}}
\newcommand{\hmx}[3][\dummystr]{\ifthenelse{\equal{#1}{\dummystring}}{\hmat{#2}_{#3}}{\hmat{#2}_{#3}^{(#1)}}}

\DeclareMathOperator{\diag}{diag}
\DeclareMathOperator*{\argmax}{arg max}
\DeclareMathOperator*{\argmin}{arg min}

\newcommand{\Ex}{\mathrm{E}}
\newcommand{\Var}{\mathrm{Var}}
\newcommand{\Cov}{\mathrm{Cov}}

\newcommand{\Ind}[1]{\mathrm{I}_{\{#1\}}}

\newcommand{\Id}{\mat{I}}

\newcommand{\eps}{\mathrm{\varepsilon}}

\newcommand{\figref}[1]{Figure~\ref{fig:#1}}

\newcommand{\secref}[1]{Section~\ref{sec:#1}}
\renewcommand{\eqref}[1]{Eq.~\ref{eq:#1}}
\newcommand{\eqp}[1]{(\ref{eq:#1})}
\newcommand{\algref}[1]{Algorithm~\ref{alg:#1}}


%% file: mymacros2.tex
\newcommand{\vb}[2][\dummystring]{\vci[#1]{b}{#2}}

\newcommand{\vd}[2][\dummystring]{\vci[#1]{d}{#2}}

\newcommand{\vf}[2][\dummystring]{\vci[#1]{f}{#2}}
\newcommand{\vg}[2][\dummystring]{\vci[#1]{g}{#2}}

\newcommand{\vs}[2][\dummystring]{\vci[#1]{s}{#2}}

\newcommand{\vu}[2][\dummystring]{\vci[#1]{u}{#2}}
\newcommand{\vv}[2][\dummystring]{\vci[#1]{v}{#2}}

\newcommand{\vy}[2][\dummystring]{\vci[#1]{y}{#2}}
\newcommand{\vz}[2][\dummystring]{\vci[#1]{z}{#2}}

\newcommand{\vdelta}[2][\dummystring]{\vci[#1]{\delta}{#2}}
\newcommand{\veps}[2][\dummystring]{\vci[#1]{\eps}{#2}}
\newcommand{\vth}[2][\dummystring]{\vci[#1]{\theta}{#2}}
\newcommand{\vmu}[2][\dummystring]{\vci[#1]{\mu}{#2}}

\newcommand{\vpi}[2][\dummystring]{\vci[#1]{\pi}{#2}}

\newcommand{\vrho}[2][\dummystring]{\vci[#1]{\rho}{#2}}

\newcommand{\tvb}[2][\dummystring]{\tvci[#1]{b}{#2}}

\newcommand{\tvf}[2][\dummystring]{\tvci[#1]{f}{#2}}

\newcommand{\tvy}[2][\dummystring]{\tvci[#1]{y}{#2}}

\newcommand{\tvth}[2][\dummystring]{\tvci[#1]{\theta}{#2}}

\newcommand{\tvpi}[2][\dummystring]{\tvci[#1]{\pi}{#2}}

\newcommand{\mxa}[2][\dummystring]{\mx[#1]{A}{#2}}
\newcommand{\mxb}[2][\dummystring]{\mx[#1]{B}{#2}}

\newcommand{\mxf}[2][\dummystring]{\mx[#1]{F}{#2}}

\newcommand{\mxl}[2][\dummystring]{\mx[#1]{L}{#2}}

\newcommand{\mxx}[2][\dummystring]{\mx[#1]{X}{#2}}

\newcommand{\rng}[2][1]{{#1},\dots,{#2}}

\newcommand{\srng}[2][1]{\{{#1},\dots,{#2}\}}

%% file: fastep_arxiv.bbl
\begin{thebibliography}{10}

\bibitem{Barber:06}
D.~Barber.
\newblock Expectation correction for smoothing in switching linear {Gaussian}
  state space models.
\newblock {\em Journal of Machine Learning Research}, 7:2515--2540, 2006.

\bibitem{Combettes:10}
P.~Combettes and J.~Pesquet.
\newblock Proximal splitting methods in signal processing.
\newblock In H.~Bauschke, R.~Burachik, P.~Combettes, V.~Elser, D.~Luke, and
  H.~Wolkowicz, editors, {\em Fixed-Point Algorithms for Inverse Problems in
  Science and Engineering}. Springer, 2010.

\bibitem{Gerven:10}
M.~van Gerven, B.~Cseke, F.~de Lange, and T.~Heskes.
\newblock Efficient {Bayesian} multivariate {fMRI} analysis using a sparsifying
  spatio-temporal prior.
\newblock {\em Neuroimage}, 50:150--161, 2010.

\bibitem{Gerwinn:08}
S.~Gerwinn, J.~Macke, M.~Seeger, and M.~Bethge.
\newblock Bayesian inference for spiking neuron models with a sparsity prior.
\newblock In Platt et~al. \cite{NIPS-20:08}.

\bibitem{Kuss:05}
M.~Kuss and C.~Rasmussen.
\newblock Assessing approximate inference for binary {Gaussian} process
  classification.
\newblock {\em Journal of Machine Learning Research}, 6:1679--1704, 2005.

\bibitem{Levin:09}
A.~Levin, Y.~Weiss, F.~Durand, and W.~Freeman.
\newblock Understanding and evaluating blind deconvolution algorithms.
\newblock In {\em Computer Vision and Pattern Recognition}, 2009.

\bibitem{MacKay:03}
D.~MacKay.
\newblock {\em Information Theory, Inference, and Learning Algorithms}.
\newblock Cambridge University Press, 2003.

\bibitem{Minka:01a}
T.~Minka.
\newblock Expectation propagation for approximate {Bayesian} inference.
\newblock In J.~Breese and D.~Koller, editors, {\em Uncertainty in Artificial
  Intelligence 17}. Morgan Kaufmann, 2001.

\bibitem{Nickisch:08}
H.~Nickisch and C.~Rasmussen.
\newblock Approximations for binary {Gaussian} process classification.
\newblock {\em Journal of Machine Learning Research}, 9:2035--2078, 2008.

\bibitem{Nickisch:09}
H.~Nickisch and M.~Seeger.
\newblock Convex variational {Bayesian} inference for large scale generalized
  linear models.
\newblock In L.~Bottou and M.~Littman, editors, {\em International Conference
  on Machine Learning 26}, pages 761--768. Omni Press, 2009.

\bibitem{Opper:01}
M.~Opper and O.~Winther.
\newblock Adaptive and self-averaging {Thouless}-{Anderson}-{Palmer} mean field
  theory for probabilistic modeling.
\newblock {\em Physical Review E}, 64(056131), 2001.

\bibitem{Opper:05}
M.~Opper and O.~Winther.
\newblock Expectation consistent approximate inference.
\newblock {\em Journal of Machine Learning Research}, 6:2177--2204, 2005.

\bibitem{NIPS-20:08}
J.~Platt, D.~Koller, Y.~Singer, and S.~Roweis, editors.
\newblock {\em Advances in Neural Information Processing Systems 20}. Curran
  Associates, 2008.

\bibitem{Rockafellar:70}
R.~Rockafellar.
\newblock {\em Convex Analysis}.
\newblock Princeton University Press, 1970.

\bibitem{Seeger:07d}
M.~Seeger.
\newblock Bayesian inference and optimal design for the sparse linear model.
\newblock {\em Journal of Machine Learning Research}, 9:759--813, 2008.

\bibitem{Seeger:10a}
M.~Seeger.
\newblock Gaussian covariance and scalable variational inference.
\newblock In J.~F\"{u}rnkranz and T.~Joachims, editors, {\em International
  Conference on Machine Learning 27}. Omni Press, 2010.

\bibitem{Seeger:08d}
M.~Seeger and H.~Nickisch.
\newblock Large scale {Bayesian} inference and experimental design for sparse
  linear models.
\newblock To appear in {\em SIAM Journal of Imaging Sciences} (arXiv:0810.0901v2), 2010.

\bibitem{Seeger:08a}
M.~Seeger, H.~Nickisch, R.~Pohmann, and B.~Sch\"{o}lkopf.
\newblock {Bayesian} experimental design of magnetic resonance imaging
  sequences.
\newblock In D.~Koller, D.~Schuurmans, Y.~Bengio, and L.~Bottou, editors, {\em
  Advances in Neural Information Processing Systems 21}, pages 1441--1448.
  Curran Associates, 2009.

\bibitem{Wipf:08}
D.~Wipf and S.~Nagarajan.
\newblock A new view of automatic relevance determination.
\newblock In Platt et~al. \cite{NIPS-20:08}, pages 1625--1632.

\end{thebibliography}
